\documentclass[letterpaper, 10 pt, conference]{ieeeconf}  %

\IEEEoverridecommandlockouts                              %

\usepackage{cite}
\usepackage{amsmath,amssymb,amsfonts,thmtools,mathrsfs}
\usepackage{graphicx}
\usepackage{calc}
\usepackage{textcomp}
\usepackage{comment}
\newlength{\oldlabelindent}
\makeatletter
\let\labelindent\@undefined
\makeatother
\usepackage{enumitem}
\usepackage{subcaption}
\usepackage{siunitx}
\usepackage{svg}
\usepackage{siunitx}
\usepackage[colorinlistoftodos]{todonotes}
\usepackage[colorlinks=false,]{hyperref}
\usepackage[capitalise,nameinlink]{cleveref}
\usepackage{algorithm}
\usepackage{algpseudocode}
\usepackage{color}
\usepackage{xcolor}
\usepackage{booktabs}
\usepackage{multirow}
\usepackage{svg}

\declaretheorem[name=Theorem,numberwithin=section]{theorem}
\declaretheorem[name=Proposition,numberwithin=section]{proposition}
\declaretheorem[name=Lemma,numberwithin=section]{lemma}
\declaretheorem[name=Corollary,numberwithin=section]{corollary}

\declaretheorem[name=Definition,numberwithin=section]{definition}

\DeclareMathOperator{\qmin}{HMin_\gamma}
\DeclareMathOperator{\qmax}{HMax_\gamma}

\newlist{definitionlist}{enumerate}{1}
\newlist{corollarylist}{enumerate}{1}
\newlist{lemmalist}{enumerate}{1}
\newlist{theoremlist}{enumerate}{1}
\newlist{propositionlist}{enumerate}{1}
\newlist{feature}{enumerate}{1}
\setlist[corollarylist]{label=\roman{corollarylisti}, ref=\thecorollary.\roman{corollarylisti},noitemsep}
\setlist[definitionlist]{label=\roman{definitionlisti}, ref=\thedefinition.\roman{definitionlisti},noitemsep}
\setlist[propositionlist]{label=\roman{propositionlisti}, ref=\theproposition.\roman{propositionlisti},noitemsep}
\setlist[lemmalist]{label=\roman{lemmalisti}, ref=\thelemma.\roman{lemmalisti},noitemsep}
\setlist[theoremlist]{label=\roman{theoremlisti}, ref=\thetheorem.\roman{theoremlisti},noitemsep}
\setlist[feature]{label=(\alph*), ref=(\alph*),noitemsep}

\crefname{theorem}{Theorem}{Theorems}
\crefname{proposition}{Proposition}{Propositions}
\crefname{lemma}{Lemma}{Lemmas}
\crefname{definition}{Definition}{Definitions}
\crefname{corollary}{Corollary}{Corollaries}
\crefname{listprop}{Property}{Properties}
\crefname{example}{Example}{Examples}
\crefname{remark}{Remark}{Remarks}
\crefname{assumption}{Assumption}{Assumptions}
\crefname{claim}{Claim}{Claims}
\crefname{conjecture}{Conjecture}{Conjectures}
\crefname{fact}{Fact}{Facts}
\crefname{problem}{Problem}{Problems}
\crefname{solution}{Solution}{Solutions}
\crefname{criterion}{Criterion}{Criteria}

\crefname{listprop}{Property}{Properties}
\addtotheorempostheadhook[corollary]{\crefalias{corollarylisti}{corollary}}
\addtotheorempostheadhook[definition]{\crefalias{definitionlisti}{definition}}
\addtotheorempostheadhook[lemma]{\crefalias{lemmalisti}{lemma}}
\addtotheorempostheadhook[theorem]{\crefalias{theoremlisti}{theorem}}
\addtotheorempostheadhook[proposition]{\crefalias{propositionlisti}{proposition}}
\crefname{featurei}{Feature}{Features}
\crefname{figure}{Fig.}{Figs.}

\DeclareMathOperator{\sgn}{sgn}

\makeatletter
\newcommand*{\bigboxplus}{%
  \DOTSB
  \mathop{\vphantom{\bigoplus}\mathpalette\matt@bigboxplus\relax}%
  \slimits@
}
\newcommand\matt@bigboxplus[2]{%
  \vcenter{\m@th\hbox{\resizebox{\widthof{$#1\bigoplus$}}{!}{$\boxplus$}}}%
}
\makeatother

\renewcommand{\proof}[1][]{%
  \noindent\hspace{2em}%
  {\itshape Proof\if\relax\detokenize{#1}\relax: \else\ of #1: \fi}%
}

\newcommand{\qedsymbol}{\hfill$\blacksquare$}

\title{\LARGE \bf
Hölder Signed Distance: A Differentiable, Signed, Parallelizable Metric for Robotics
}

\author{Felipe Bartelt, Ali Umut Kaypak, Anthony Tzes, Farshad Khorrami,\\Luciano C. A. Pimenta, Vinicius M. Gon\c{c}alves
\thanks{}%
}

\begin{document}

\maketitle
\thispagestyle{empty}
\pagestyle{empty}

\begin{abstract}
Computing distances between sets is essential in robotic motion planning and control, where differentiable gradients enable real-time optimization. The Euclidean Signed Distance Function (SDF), however, is not differentiable everywhere, and existing alternatives often sacrifice differentiability, sign information, or computational efficiency. In this letter, we introduce a novel differentiable signed distance between convex polyhedra. To this end, we first propose differentiable versions of the minimum and maximum operators, termed the \emph{Hölder minimum} and \emph{Hölder maximum}. We then replace the original min-max operators in the classical SDF formulation, yielding the \emph{Hölder signed distance}. Unlike prior differentiable distance formulations that rely on iterative algorithms, our approach is computed in closed form, eliminating convergence issues while remaining naturally amenable to GPU parallelization. We validate the practical advantages and computational performance of the proposed distance through runtime comparisons with existing approaches. We also present a robotic manipulator experiment, demonstrating its suitability for applications in control.
\end{abstract}

\section{Introduction}
Efficient and safety-aware robot operation relies heavily on computing distances between sets for motion planning, collision detection, and control. Many safe‑control techniques require not only the distance value but also its derivatives. Control barrier functions (CBFs), for instance, heavily rely on the gradient or even higher derivatives of a distance-like function to enforce safety constraints \cite{Singletary2021,Singletary2022,Thirugnanam2022cbf,Thirugnanam2022,Thirugnanam2023,Ferreira2026}. Failing to comply with this differentiability property when using CBFs, for instance, may cause practical issues like very large control inputs \cite{Goncalves2024}, or very noisy control inputs \cite{Goncalves2026}. 

The traditional Euclidean distance function, however, is not differentiable everywhere \cite{Escande2014}. Non-differentiability usually appears when the \emph{witness points} (the pair of points, one on each object, that attains the minimum distance) are not unique. \Cref{fig:square-example-hdsdf-sdf} illustrates such non‑differentiability for two moving squares. Overcoming this differentiability problem has motivated two main strategies: (i) restricting geometry to shapes with well‑defined Euclidean distance derivatives \cite{Escande2007,Escande2014,Tracy2022}, or (ii) adopting alternative metrics that are differentiable by construction \cite{Xu2020,Tracy2023,Goncalves2024gpu,Wei2024, Goncalves2024, Goncalves2026} (see the review in \cite{Goncalves2024}).

Following the second strategy, \cite{Goncalves2024} introduced a smooth distance between convex sets, and applied to second‑order kinematic control. Their method generalized von Neumann's alternating projection algorithm to an iterative distance computation, but it remains cumbersome to evaluate and does not guarantee that the metric vanishes (or becomes negative) when the sets overlap. \cite{Goncalves2026} removed these drawbacks, delivering a far simpler metric that correctly vanishes upon overlap. More recently, \cite{Nunes2026} developed a differentiable metric between a point and a 1D curve and used it to construct Guiding Vector Fields for smooth path following.
\begin{figure}[t]
    \centering
    \includegraphics[trim=4px 8px 6px 4px, clip, width=\columnwidth]{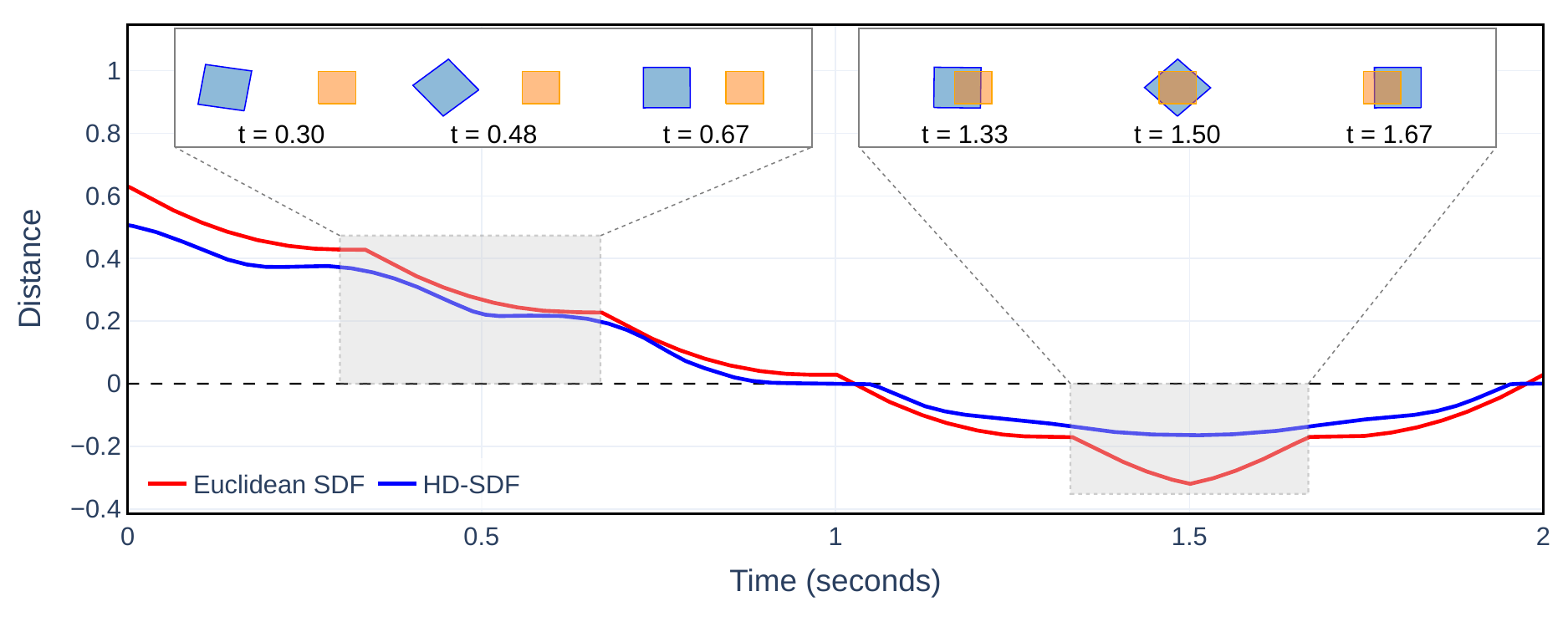}
    \caption{A square translates and rotates past another. The Euclidean SDF exhibits non‑differentiable kinks (highlighted with snapshots) where witness points or separating normals are non‑unique; our HD‑SDF remains smooth throughout and never signals false separation.}
    \label{fig:square-example-hdsdf-sdf}
\end{figure}

All the aforementioned metrics are \emph{unsigned}: they are only nonnegative and do not indicate penetration depth. While unsigned distances are sufficient for avoiding collisions, many safety‑critical applications require a \emph{Signed Distance Function} (SDF) that distinguishes separation from penetration. CBF formulations, for example, can directly enforce non‑penetration using a signed distance \cite{Chen2026,Wu2025}. In planning, a signed distance provides penetration depth when collisions are unavoidable or when guiding a robot out of an already violated region. Therefore, a differentiable distance that is both signed and non‑iterative would serve a broad spectrum of robotic tasks. Furthermore, optimization‑based motion planners, such as CHOMP \cite{Ratliff2009}, similarly need SDF gradients to deform trajectories away from obstacles. Without analytic gradients, practitioners resort to discretized distance fields \cite{Gao2018,Han2019}, smoothing approximations \cite{Maric2024}, or surrogate obstacle potentials \cite{Zhou2021}. While effective, these workarounds can compromise geometric fidelity near obstacle boundaries, leading to conservative paths or unreliable gradient information. Thus, an analytic differentiable distance, even unsigned, is highly attractive for planning as well.

This work improves on \cite{Goncalves2026} providing a \emph{signed distance} between two objects that is differentiable everywhere. The core idea is to take the classic min‑max formulation of the Euclidean SDF and replace the hard minimum and maximum operators with novel differentiable substitutes that we introduce, together with a few additional modifications. The benefit of our approach is already visible in \Cref{fig:square-example-hdsdf-sdf}, where the HD‑SDF remains smooth at configurations that cause non‑differentiability of the Euclidean SDF. The \textbf{contributions} of this paper are: 
\begin{itemize}
    \item We present almost‑differentiable minimum and maximum operators: the \emph{H\"older Minimum and Maximum}, denoted $\qmin$ and $\qmax$, that preserve key properties of the classical min and max (e.g., associativity, commutativity) while being differentiable everywhere except on a zero‑measure set, where a simple nonlinear function restores differentiability (as opposed to the traditional min and max).
    
    \item Using these operators, we construct and analyze a differentiable signed distance function $D_{\gamma,\epsilon}(\mathcal{A},\mathcal{B})$ between two convex polyhedra $\mathcal{A}$, $\mathcal{B}$, that we call \emph{H\"older Differentiable Signed Distance Function}. The distance is $\gamma$-times differentiable, with $\gamma$ controlling the order of differentiability. In contrast to \cite{Goncalves2024, Goncalves2026}, this metric is \emph{signed}, providing penetration depth. Moreover, unlike the iterative schemes of \cite{Goncalves2024, Goncalves2026}, our distance is computed in closed form, avoiding convergence issues and enabling straightforward parallelization on GPUs. The metric is \emph{conservative}: it may yield a non‑positive value even in the absence of collision, but such false positives are rare and occur only when the objects are near contact; the degree of conservativeness can be traded for extra computation.
    
    \item We demonstrate the practical advantages of our distance in simulations and experiments, comparing it with the exact SDF and the differentiable distance from \cite{Goncalves2026}.
\end{itemize}

 For readability, all proofs and auxiliary lemmas are deferred to Appendix \ref{appendix:proofs}.

\subsection{Notation}
Regarding notation, we adopt $\mathbb{R}_+$ for the nonnegative real numbers and $\mathbb{R}_-$ for the nonpositive real numbers. Every vector in this work must be regarded as a column vector, and the gradient $\nabla f$ of a scalar function $f$ is a row vector. The norm $\|\cdot\|$ represents the Euclidean norm. For a polyhedron $\mathcal{A}$, $\mathcal{V}(\mathcal{A})$ is the set of its vertices. Given two polyhedra $\mathcal{A}$ and $\mathcal{B}$, $\mathcal{A}-\mathcal{B} \triangleq \{a-b\;\vert\; a\in \mathcal{A},b\in\mathcal{B}\}$ is the Minkowski difference. We use $\pm \mathcal{A}$ as a shorthand for $\mathcal{A}\cup(\{0\}-\mathcal{A})$. The function $\sgn:\mathbb{R}\to\{-1,0,1\}$ is given by $\sgn(x) = 1$ if $x>0$, $\sgn(x) = 0$ if $x=0$, and $\sgn(x) = -1$ if $x<0$. Let $x\in\mathbb{R}^n$, we denote by $\|x\|_p$ the \emph{p-Hölder norm}\footnote{This naming is inspired by the Hölder mean\cite{sykora2009}; the norm differs from the mean only by a constant factor $n^{-1/p}$.}, given by $\|x\|_p \triangleq \left(\sum_i |x_i|^p\right)^{1/p}$, which is valid when $p\neq0$---meaning it is valid even when $p$ is negative\footnote{When $p<1$ this is not technically a norm, but we will not need this fact and hence we will adopt the nomenclature.}. Furthermore, we denote by $(x)_{+}$ the vector obtained from $x$ by zeroing all the negative entries and keeping all the other intact. 

\section{H\"older Max/Min Functions}\label{sec:smooth-min-max-functions}
Our work builds upon a min–max formulation of the signed Euclidean distance. The distance is defined in the same structural manner, but we replace the standard $\min$ and $\max$ operators with almost-everywhere differentiable surrogates introduced in this section, that we call \emph{$\gamma$-H\"older Minimum} ($\qmin$) and \emph{$\gamma$-H\"older Maximum} ($\qmax$), being parametrized by a scalar $\gamma > 0$ that controls the degree of differentiability. Although the classical minimum and maximum are themselves almost-everywhere differentiable, our formulation concentrates the points of non-differentiability into a more convenient set. When compared to their traditional counterparts, this structure makes it much easier to eliminate such non-differentiabilities later via a nonlinear transformation. Furthermore, the $\qmin$ and $\qmax$ are generalizations of the traditional minimum and maximum, respectively, since they can be recovered by setting the parameter $\gamma$ to infinite.

\subsection{H\"older Minimum}\label{sec:holder-min}

We start by defining the \emph{H\"older Minimum}\footnote{The name arises from the Hölder norm representation in \Cref{lemma:aczel}.} (referred to as $\qmin$), which generalizes the usual minimum:
\begin{definition}\label{def:smooth-min-operator}

 For an integer $\gamma \geq 1$ and two scalars $x$ and $y$, we define the function $\mu_\gamma:\mathbb{R}_+\times\mathbb{R}_+\to\mathbb{R}_+$ as\footnote{Note that formally $\mu_\gamma(x, 0)$ or $\mu_\gamma(0, y)$ are not defined. In that case, one can easily show that the limits exist, and we can set $\mu_\gamma(x, 0) \triangleq 0$, $\mu_\gamma(0, y) \triangleq 0$ for all $x, y \geq 0$.}
    \begin{align*}
        \mu_\gamma(x, y) \triangleq \bigl(x^{-(\gamma+1)} + y^{-(\gamma+1)}\bigr)^{-1/(\gamma+1)}.
    \end{align*}
    
    Let $x, y\in\mathbb{R}$, then the \emph{$\gamma$-H\"older Minimum Operator} ($\qmin$) $\oplus:\mathbb{R}\times\mathbb{R}\to\mathbb{R}$ is defined as
    \begin{align*}
        x\oplus y \triangleq \begin{cases}
            \hspace{8pt} \mu_\gamma(x, y), & x,y\ge0,\\
        -\mu_\gamma(-x^{-1}, -y^{-1})^{-1}, & x,y<0,\\
        \hspace{8pt} \min(x, y),& \text{otherwise}.\tag*{\qedsymbol}
        \end{cases}
    \end{align*}
\end{definition}
One can check by inspection that, when $\gamma \rightarrow \infty$, the traditional minimum is recovered.

Although this operation has a piecewise definition, it turns out to be very amenable to differentiation. More precisely, it is differentiable almost everywhere (as \Cref{lem:smooth-min-differentiable} below shows), and the points of non-differentiability are easily removable by the application of a non-linear function in the output (as \Cref{lemma:diffcompos} below shows). This fact will be essential in the formulation of our differentiable signed distance function.

We now derive some properties of the $\qmin$ which are shared with the usual minimum.
\begin{proposition}[Basic properties of $\qmin$]\label{lem:smooth-min-oplus-basic-properties}
    Let $x,y,z$ be scalars and $\{x_i\}_{i=1}^n$ be a set of scalars. The $\oplus$ operator $\qmin$ enjoys the following properties
    \begin{propositionlist}
        \item Commutativity: $x \oplus y = y \oplus x$;\label{property:smooth-min-oplus-commutative}
        \item Associativity: $(x\oplus y)\oplus z = x\oplus (y\oplus z)$, which allows the shorthand notation $\bigoplus_{i} x_i$ for the $\qmin$ of a set $\{x_i\}_{i=1}^n$;\label{property:smooth-min-oplus-associative}
        \item Sign preservation: $\sgn\bigl(\min_i x_i\bigl) = \sgn\bigl(\bigoplus_{i} x_i\bigr)$. \label{property:smooth-min-oplus-sign-preservation}
        \item Limit case: when $\gamma \rightarrow +\infty$, $\bigoplus_i x_i = \min_i x_i$. \qedsymbol
        \label{property:smooth-min-limit}
    \end{propositionlist}
\end{proposition}

As $\bigoplus$ is a multivariate function, a formal definition of differentiability is required before any discussion on its differentiability.
\begin{definition}
    A function $f: \mathbb{R}^n \to \mathbb{R}$ is said $k$ times differentiable in an open set $\mathcal{A} \subset \mathbb{R}^n$ if for any $(x_1,\dots,x_n) \in \mathcal{A}$ and any $k_1,\dots,k_n$ such that $k = \sum_i k_i$ , the following $k^{th}$ order mixed partial derivative is is continuous in $\mathcal{A}$:
    \begin{align*}
        \frac{\partial^{k} f}{\partial x_1^{k_1} \dots \partial x_n^{k_n}}(x_1,\dots,x_n).\tag*{\qedsymbol}
    \end{align*}
\end{definition}

With this definition, we can state the next Lemma, that shows that $f(x) = \bigoplus_{i=1}^n x_i$ is differentiable everywhere except (possibly) when $\min_i x_i=0$, which by \cref{property:smooth-min-oplus-sign-preservation} is equivalent to $f(x) =0$.
\begin{proposition}\label{lem:smooth-min-differentiable}
    Let $\oplus$ be the $\qmin$. The function $f(x) = \bigoplus_{i=1}^n x_i$ is continuous everywhere, and it is $\gamma$-times differentiable\footnote{The condition $\gamma \geq 1$ in \Cref{def:smooth-min-operator} guarantees that the operator is at least once differentiable} on the open set $\mathcal{D}=\{x \in \mathbb{R}^n \ \vert \ \min_i x_i \neq 0\} = \{x \in \mathbb{R}^n \ \vert \ f(x) \neq 0\}$. \hfill\qedsymbol
\end{proposition}

Although $f(x) = \bigoplus_{i=1}^{n} x_i$ may be non-differentiable when $f(x) = 0$, one can apply simple nonlinear transformation $\phi$ to obtain a new function $g(x) = \phi(f(x))$ that is differentiable everywhere, while largely preserving the shape of $f$. As a motivating example, let $\phi(u) = u^k$ with $k$ odd. The first partial derivative with respect to $x_1$ is $k f(x)^{k-1} \frac{\partial f}{\partial x_1}(x)$. 
At a point $x^*$ where $f(x^*)=0$, differentiability of $g$ can still be achieved if the factor $f(x)^{k-1}$ forces the product to vanish as $x\to x^*$, i.e., if $f$ goes to $0$ faster than any potential blow-up of $\frac{\partial f}{\partial x_1}(x)$.

In contrast, the true minimum $f_{\mathrm{min}}(x) = \min_{i=1,..,n} x_i$ is non-differentiable whenever the minimum is attained by more than one argument. For such a function it is far more difficult (or even impossible) to design a transformation $\phi$ that restores differentiability while preserving shape without trivializing the output (e.g. $\phi(u) = 0$ for all $u$). This illustrates a key advantage of our $\qmin$: all non-differentiable points are confined to the zero level set $f(x)=0$, a much smaller and better structured set.

The choice $\phi(u) = u^k$ with odd $k$ can indeed remove the non-differentiability, but it also distorts the function significantly. We now formally define a \emph{shaping} function that eliminates the non‑differentiability while preserving the original shape far more faithfully.
\begin{definition}
\label{def:ordershap}
    A $k^{th}$ \emph{order shaping function} $\phi$ is a function $\phi: \mathbb{R} \to \mathbb{R}$ satisfying:
    \begin{definitionlist}
        \item $\phi$ is $k$-times differentiable.\label{property:shaping-differentiable}
        \item $\lim_{u \to 0} |\phi(u)/u^r|$ is zero for all integers $r \leq k$, and for $r = k+1$ it is finite and positive.\label{property:shaping-limit-zero}
        \item $\phi$ is strictly increasing.\label{property:shaping-increasing}
        \item $\lim_{u \to \pm \infty} (\phi(u)-u) = 0$.\label{property:shaping-limit-large-input} \hfill\qedsymbol
    \end{definitionlist}
\end{definition}
\Cref{property:shaping-limit-zero} provides enough flatness to smooth out non-differentiabilities while \Cref{property:shaping-limit-large-input} guarantees that $\phi$ approximates the identity for large $|u|$, preserving the original function shape away from zero. An example of a shaping function is
    \begin{equation}
    \label{eq:phi}
    \Phi_{k,\epsilon}(u) \triangleq u\frac{|u|^{k}}{|u|^{k}+\epsilon}
\end{equation}
with $\epsilon > 0$. Note that for a small $\epsilon$, $\Phi_{k,\epsilon}(u)$ closely matches the identity function. We then can discuss the differentiability of $\qmin$ under the application of a shaping function $\phi$.

\begin{proposition}\label{lemma:diffcompos}
    Let $\oplus$ be the $\qmin$. Let $\phi: \mathbb{R} \to \mathbb{R}$ be a $\gamma^{th}$ order shaping function. Let $f(x) = \bigoplus_{i=1}^n x_i$. Then,  the function $g(x) = \phi\big(f(x)\big)$ is $\gamma$-times differentiable everywhere, with mixed partial derivatives of order $0 \leq k \leq \gamma$ such that  $f^{(k)}(x)=0$  when $f(x)=0$.\hfill\qedsymbol
\end{proposition}

We will discuss how the $\qmin$ and this $\phi$-shaping procedure can be useful for the purpose of obtaining a differentiable distance in \Cref{sec:distance-function}.

\subsection{H\"older Maximum}
We define the counterpart of the $\qmin$ operator in terms of \cref{def:smooth-min-operator} as follows:
\begin{definition}\label{def:smooth-max-boxplus}
    Let $x, y \in \mathbb{R}$ and $\oplus$ the $\qmin$ operator. The \emph{$\gamma$-H\"older Maximum Operator} (denoted by $\qmax$) $\boxplus:\mathbb{R}\times\mathbb{R}\to\mathbb{R}$ is defined as
    \begin{align*}
        x\boxplus y \triangleq -((-x)\oplus(-y)).\tag*{\qedsymbol}
    \end{align*}%
\end{definition}
This operator inherits the properties of the $\qmin$ $\oplus$ as the next corollary shows.
\begin{corollary}[Properties of $\qmax$]\label{cor:smooth-max-properties}
    Let $x,y,z$ be scalars and $\{x_i\}_{i=1}^{n}$ be a set of scalars. The $\qmax$  operator $\boxplus$ enjoys the following properties, each derived directly from the corresponding lemma for $\oplus$.
    \begin{corollarylist}
        \item Commutativity: $x\boxplus y = y\boxplus x$; \label{property:smooth-max-commutative}
        \item Associativity: $(x\boxplus y)\boxplus z = x\boxplus (y\boxplus z)$, allowing the shorthand $\bigboxplus_{i=1}^{n}x_i$;\label{property:smooth-max-associative}
        \item Sign Preservation: $\sgn\bigl(\max_i x_i\bigr) = \sgn\bigl(\bigboxplus_{i} x_i\bigr)$; \label{property:smooth-max-sign-preservation}
        \item Differentiability: if $\phi$ is a $\gamma$-th order shaping function (\Cref{def:ordershap}), $f(x) = \phi\left(\bigboxplus_{i} x_i\right)$ is $\gamma-$times differentiable everywhere;\label{property:smooth-max-differentiable}
        \item Limit case: when $\gamma \rightarrow +\infty$, $\bigboxplus_i x_i = \max_i x_i$.
    \end{corollarylist}
\end{corollary}
\begin{proof}
    Each statement follows by using the negated inputs $-x_i$, the fact that $\min_i-x_i=-\max_i x_i$ and invoking \cref{lem:smooth-min-oplus-basic-properties,lemma:diffcompos}.
\end{proof}

\section{H\"older Differentiable Signed Distance Function}\label{sec:distance-function}
In this section, we introduce our differentiable signed distance function between convex polyhedra, which is rooted in a max/min formulation of the Euclidean signed distance between convex sets:
\begin{theorem}[SDF, adapted from \cite{Wu2025}]
     Let $\mathcal{A}, \mathcal{B}\subset\mathbb{R}^m$ be two bounded convex sets. 
    The Euclidean signed distance between $\mathcal{A}$ and $\mathcal{B}$ is given by
    \begin{align}
        d(\mathcal{A},\mathcal{B})= \max_{\|n\|=1}\min_{\substack{a\in \mathcal{A}\\b\in \mathcal{B}}}n^\top (a-b). \tag*{\qedsymbol}
    \end{align}
\end{theorem}

Although this result holds true for general convex sets $\mathcal{A}$ and $\mathcal{B}$, henceforth we will focus on \emph{polyhedra}. 

We will make a first step towards a differentiable distance by replacing the set $\{n\in\mathbb{R}^m\;\vert\;\|n\|=1\}$ by a finite set. The resulting expression will not be exactly equal to $d$, but will preserve the same \emph{sign} (i.e., applying $\sgn$ on both renders the same result). For this, we use the \emph{Separating Axis Theorem (SAT)~\cite{vanDenBergen2003}} for convex polyhedra in three dimensions. To state this result, we need the following definition.

\begin{definition}[Candidate Normals Set]\label{def:normal-set}
    Let $\mathcal{A},\mathcal{B}\subset\mathbb{R}^3$ be convex polyhedra. Let $\mathcal{F}(\mathcal{A})$ denote the set of outward unit normal vectors of the faces of $\mathcal{A}$, and let $\mathcal{E}(\mathcal{A})$ denote the set of unit direction vectors of the edges of $\mathcal{A}$, where for each edge exactly one of the two possible orientations is chosen arbitrarily. Define $\mathcal{F}(\mathcal{B})$ and $\mathcal{E}(\mathcal{B})$ analogously for $\mathcal{B}$. The set of normals $\mathcal{N}(\mathcal{A}, \mathcal{B})$ is then
    \begin{align*}
        \mathcal{N}(\mathcal{A}, \mathcal{B}) \triangleq \pm\bigl(\mathcal{F}(\mathcal{A})\cup\mathcal{F}(\mathcal{B}) \cup \mathcal{N}_E(\mathcal{A}, \mathcal{B})\bigr),
    \end{align*}
    where
    \begin{align*}
        \mathcal{N}_E(\mathcal{A}, \mathcal{B}) &\triangleq \biggl\{\frac{u\times v}{\|u\times v\|}\ \Big\vert\ u\in \mathcal{E}(\mathcal{A}), v\in\mathcal{E}(\mathcal{B})\biggr\}. \tag*{\qedsymbol}
    \end{align*}
\end{definition}
With this set, we may replace the continuous maximization over the unit sphere by a discrete maximization over the finite set $\mathcal{N}(\mathcal{A}, \mathcal{B})$, obtaining a new metric $D(\mathcal{A}, \mathcal{B})$:
\begin{align}
\label{eq:D}
    D(\mathcal{A},\mathcal{B}) \triangleq \max_{n\in\mathcal{N}(\mathcal{A}, \mathcal{B})}\min_{\substack{a\in \mathcal{A}\\b\in \mathcal{B}}}n^\top (a-b).
\end{align}
With this, we can state the \emph{Separating Axis Theorem (SAT)}.
\begin{theorem}[SAT, adapted from \cite{vanDenBergen2003}]
     Let $\mathcal{A}, \mathcal{B}\subset\mathbb{R}^m$ be two convex polyhedra. Then, if $\mathcal{A} \cap \mathcal{B} = \emptyset$, there is a plane with normal $n \in \mathcal{N}(\mathcal{A},\mathcal{B})$ that separates $\mathcal{A}$ and $\mathcal{B}$. Equivalently, $\exists\, n \in \mathcal{N}$ such that $\min_{c \in \mathcal{A}-\mathcal{B}} n^{\top}c > 0$. \hfill\qedsymbol
\end{theorem}
And then we can prove the following:
\begin{proposition}\label{lem:distance-sign-preservationB}
    Let $\mathcal{A},\mathcal{B}\subset\mathbb{R}^3$ be convex polyhedra, then
    $\sgn\bigl(D(\mathcal{A}, \mathcal{B})\bigr)=\sgn\bigl(d(\mathcal{A}, \mathcal{B})\bigr)$.\hfill\qedsymbol
\end{proposition}

Because the objectives are linear, the minimum of $n^{\top}a$ over the convex polyhedron $\mathcal{A}$ is attained at a vertex of $\mathcal{A}$, and similarly the maximum of $n^{\top}b$ over $\mathcal{B}$ is attained at a vertex of $\mathcal{B}$ (this is the \emph{Fundamental Theorem of Linear Programming}, see \cite{Bertsekas2009}). Hence the inner minimization can be restricted to the finite vertex sets $\mathcal{V}(\mathcal{A})$ and $\mathcal{V}(\mathcal{B})$, so
\begin{align}
\label{eq:Dfinal}
    D(\mathcal{A},\mathcal{B})= \max_{n \in \mathcal{N}(\mathcal{A},\mathcal{B})}\min_{\substack{a\in \mathcal{V}(\mathcal{A})\\b\in \mathcal{V}(\mathcal{B})}}n^\top (a - b).
\end{align}
We now construct a differentiable metric from $D$ by eliminating three sources of non‑differentiability in succession.

\textbf{First step, replacing $\mathcal{N}$}: Despite its sign-preservation property, the distance $D$ defined with the candidate normals set $\mathcal{N}$ suffers from a hidden discontinuity. The cross-product terms in $\mathcal{N}_E$ are normalized; when two edges become parallel, the cross product vanishes and the normalization is undefined, causing a jump in the normal set and therefore in $D$. Since our ultimate goal is a differentiable distance, we must eliminate such singularities. We therefore replace $\mathcal{N}_E$ by the set of edge direction vectors themselves, i.e., $\mathcal{E}(\mathcal{A}) \cup \mathcal{E}(\mathcal{B})$. The resulting normal set
\begin{align}
\label{eq:Ntilde}
    \tilde{\mathcal{N}}(\mathcal{A}, \mathcal{B}) \triangleq \pm\bigl(\mathcal{F}(\mathcal{A})\cup\mathcal{F}(\mathcal{B})\cup\mathcal{E}(\mathcal{A})\cup\mathcal{E}(\mathcal{B})\bigr)
\end{align}
varies continuously with the orientation of the polyhedra. Substituting $\tilde{\mathcal{N}}$ into \eqref{eq:Dfinal} gives a new metric $\tilde{D}$
\begin{align}
\label{eq:Dtilde}
     \max_{n\in\tilde{\mathcal{N}}(\mathcal{A}, \mathcal{B})}\min_{\substack{a\in \mathcal{V}(\mathcal{A})\\b\in \mathcal{V}(\mathcal{B})}}n^\top (a-b).
\end{align}
which varies continuously under rigid motions. The price is that $\tilde{D}$ no longer guarantees sign preservation; instead, it provides a \emph{conservative lower bound} of the true SDF $d$.

\textbf{Second step,  replacing min and max:} although the $\min$ and $\max$ functions are continuous, they are not differentiable. We replace them by their almost‑differentiable counterparts $\qmin$ and $\qmax$ introduced in \cref{sec:smooth-min-max-functions}, obtaining
\begin{align}
\label{eq:Dtilde2}
     \bigboxplus_{n\in\tilde{\mathcal{N}}(\mathcal{A}, \mathcal{B})}\bigoplus_{\substack{a\in \mathcal{V}(\mathcal{A})\\b\in \mathcal{V}(\mathcal{B})}}n^\top (a-b).
\end{align}
However, as stated in \Cref{lem:smooth-min-differentiable}, $\qmin$ and $\qmax$ remain non‑differentiable when their output is zero. This last source of non‑differentiability is concentrated in a single value and will be removed in the final step.

\textbf{Third step, applying shaping functions:} Shaping functions $\phi$  (\Cref{def:ordershap}) are designed precisely to eliminate the isolated non‑differentiability at zero while preserving the overall shape of the original function. We therefore apply the shaping function $\Phi_{\gamma,\epsilon}$ from \eqref{eq:phi} to both the inner Hölder minima and the outer Hölder maximum. The resulting expression, which we call the \emph{H\"older Differentiable Signed Distance Function} (HD‑SDF), is defined as follows.
\begin{definition}\label{def:set2set-signed-distance}
    The ($\gamma$,$\epsilon$) - \emph{H\"older Differentiable Signed Distance Function} (HD-SDF) between two convex polyhedra $\mathcal{A}, \mathcal{B}\subset\mathbb{R}^3$ is defined as:
    \begin{align*}
        D_{\gamma,\epsilon}(\mathcal{A},\mathcal{B}) \triangleq \Phi_{\gamma,\epsilon}\Bigg(\bigboxplus_{n\in\tilde{\mathcal{N}}(\mathcal{A}, \mathcal{B})}\Phi_{\gamma,\epsilon}\Bigg(\bigoplus_{\substack{a\in \mathcal{V}(\mathcal{A})\\b\in \mathcal{V}(\mathcal{B})}}n^\top(a-b)\Bigg)\Bigg)
    \end{align*}
    \noindent in which $\oplus$ is $\qmin$, $\boxplus$ is $\qmax$ and $\Phi_{\gamma,\epsilon}$ is the  $\gamma^{th}$ order shaping function (\Cref{def:ordershap}) in \eqref{eq:phi} for a $\epsilon > 0$. \hfill\qedsymbol
\end{definition}

The subsequent analysis only requires that $\Phi_{\gamma,\epsilon}$ is a $\gamma^{th}$-order shaping function; any alternative satisfying \Cref{def:ordershap} would work equally well. We chose \eqref{eq:phi} for its simplicity and good numerical behavior. Nevertheless, this construction inherits many desirable properties for a signed metric, and is $\gamma$-times differentiable everywhere. We will now establish these properties.

The first property is a separation safety guarantee: a positive HD-SDF guarantees strict separation of the two bodies; a non‑positive HD-SDF is a \emph{conservative} indicator of contact or penetration. This is formalized as follows.
\begin{proposition}\label{prop:holder-distance-sign-bound-of-sdf}
     Let $\mathcal{A},\mathcal{B}\subset\mathbb{R}^3$ be convex polyhedra, then
    $D_{\gamma,\epsilon}(\mathcal{A}, \mathcal{B}) >0 \implies d(\mathcal{A}, \mathcal{B})>0$ or, equivalently, $d(\mathcal{A}, \mathcal{B}) \le 0 \implies D_{\gamma,\epsilon}(\mathcal{A}, \mathcal{B}) \le 0$.\hfill\qedsymbol
\end{proposition}

In practice this conservativeness is mild: false collisions ($D_{\gamma,\epsilon}(\mathcal{A}, \mathcal{B}) \leq 0$ while $d(\mathcal{A}, \mathcal{B}) > 0$) are rare and can be reduced further by enlarging $\tilde{\mathcal{N}}$ (see \Cref{sec:experiments-conservativeness}).

To state the differentiability of the HD-SDF precisely, we must specify with respect to which variables the derivative is taken. In the robotics setting, the natural parameterization is through rigid-body motions\footnote{Applying a rigid transformation to a set means transforming each of its points individually.}: each set evolves as $\mathcal{A}(t) = R_A(t)\mathcal{A}_0+s_A(t)$ and $\mathcal{B}(t) = R_B(t)\mathcal{B}_0+s_B(t)$, with fixed polyhedra $\mathcal{A}_0$, $\mathcal{B}_0$ and time-varying rotations $Q_A,Q_B\in\mathrm{SO}(3)$ and translations $s_A, s_B\in\mathbb{R}^3$.
\begin{proposition}\label{lem:differentiability}
    Let $R_A, R_B: \mathbb{R} \to \mathrm{SO}(3)$ and $s_A, s_B : \mathbb{R} \to \mathbb{R}^3$ be $\gamma$-times differentiable functions, and let $\mathcal{A}_0$, $\mathcal{B}_0\subset\mathbb{R}^3$ be fixed convex polyhedra. Define $\mathcal{A}(t) = R_A(t)\mathcal{A}_0+s_A(t)$, $\mathcal{B}(t) = R_B(t)\mathcal{B}_0+s_B(t)$. Then the function $\lambda(t) = D_{\gamma,\epsilon}\big(\mathcal{A}(t),\mathcal{B}(t)\big)$ is $\gamma$-times differentiable with respect to $t$.\hfill\qedsymbol
\end{proposition}

We summarize the three key properties of the HD‑SDF in the following theorem.

\begin{theorem}\label{thm:holder-differentiable-distance-properties}
    The HD-SDF $D_{\gamma,\epsilon}$ satisfies the following:
    \begin{theoremlist}
        \item \textit{Separation safety:}  $D_{\gamma,\epsilon}(\mathcal{A}, \mathcal{B}) >0 \implies d(\mathcal{A}, \mathcal{B})>0$.\label{property:HD-SDF-sign-bound-separation-safety}
        \item \textit{Differentiability:} if $\mathcal{A}(t)$ and $\mathcal{B}(t)$ undergo $\gamma$-times differentiable rigid motions, then $\lambda(t) = D_{\gamma,\epsilon}\big(\mathcal{A}(t),\mathcal{B}(t)\big)$ is $\gamma$-times differentiable with respect to $t$.\label{property:HD-SDF-differentiability}
        \item \textit{Rigid transformation invariance:} for any rigid transformation $T$, $D_{\gamma,\epsilon}(T(\mathcal{A}),T(\mathcal{B})) = D_{\gamma,\epsilon}(\mathcal{A},\mathcal{B})$. \hfill\qedsymbol\label{property:HD-SDF-rigit-motion-invariance}
    \end{theoremlist}
\end{theorem}

\section{Experiments}\label{sec:experiments}

\subsection{Time comparison}
\begin{figure}[t]
    \centering
    \includegraphics[trim=8px 8px 8px 8px, clip, width=\columnwidth]{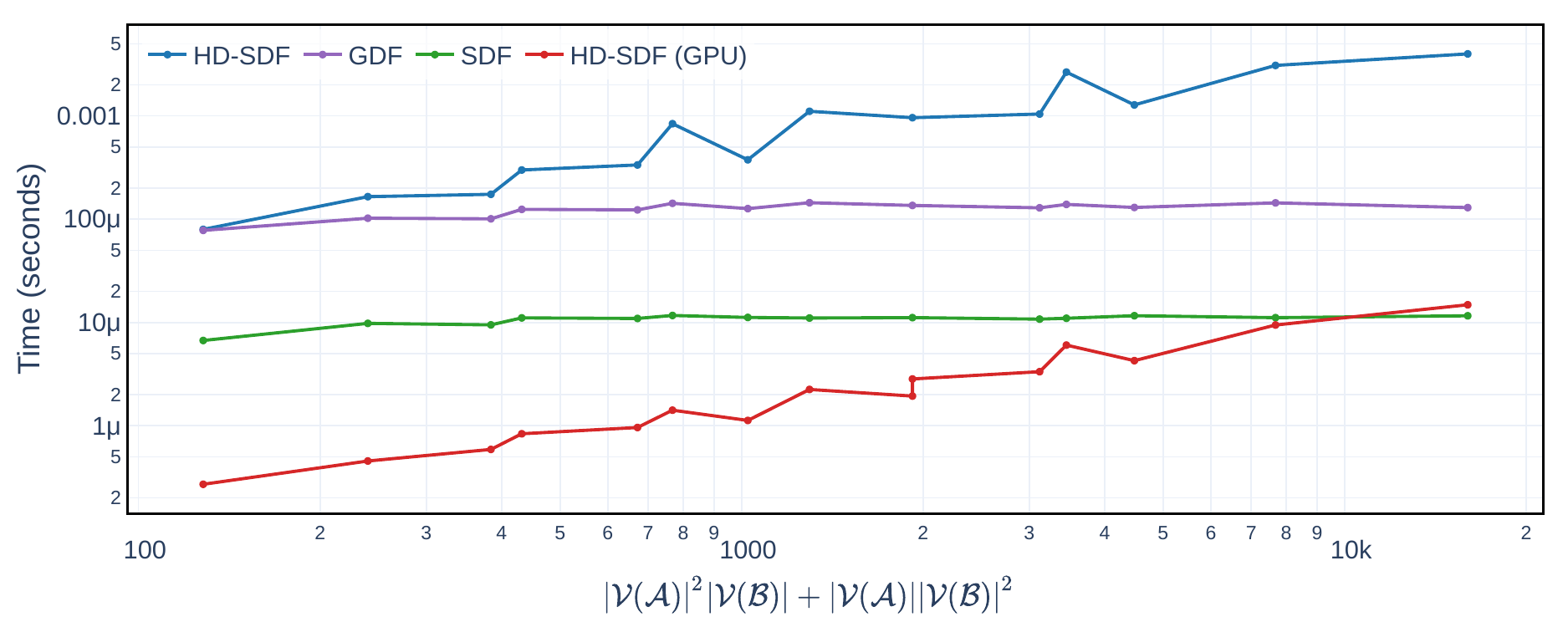}
    \caption{Average computation time on a log‑log scale for the Euclidean SDF (FCL~\cite{Pan2012}), GDF~\cite{Goncalves2026}, and our HD‑SDF executed on CPU only and with GPU acceleration. The horizontal axis is the approximate complexity $|\mathcal{V}(\mathcal{A})|^2|\mathcal{V}(\mathcal{B})|+|\mathcal{V}(\mathcal{A})||\mathcal{V}(\mathcal{B})|^2$.}
    \label{fig:slowdown-comparison}
\end{figure}
The complexity of computing an evaluation of $D_{\gamma,\epsilon}$ can be easily seen to be of order $\mathcal{O}\big(|\mathcal{V}(\mathcal{A})||\mathcal{V}(\mathcal{B})||\tilde{\mathcal{N}}(\mathcal{A},\mathcal{B})| \big)$ operations. As the numbers of faces $\mathcal{F}$ and edges $\mathcal{E}$ of a convex polyhedron grow linearly with the number of vertices~\cite{ziegler1995lectures}, $|\tilde{\mathcal{N}}(\mathcal{A},\mathcal{B})| = \mathcal{O}\big(|\mathcal{V}(\mathcal{A})|+|\mathcal{V}(\mathcal{B})|\big)$, and the total complexity becomes
\begin{align}
    \mathcal{O}\Big(|\mathcal{V}(\mathcal{A})|^2|\mathcal{V}(\mathcal{B})|+|\mathcal{V}(\mathcal{A})||\mathcal{V}(\mathcal{B})|^2 \Big), 
    \label{eq:computational-complexity}
\end{align}
that is, \emph{cubic} in the number of vertices. The evaluation consists almost entirely of independent dot products $n^{\top}(a-b)$ that can be computed in parallel, followed by associative and commutative reductions via $\qmin$ and $\qmax$. Hence the metric is naturally suited to GPU and multi‑core parallelization, and multiple object pairs can be processed simultaneously. This contrasts with the largely sequential iterative modified von Neumann procedure of \cite{Goncalves2026} and with algorithms for the exact Euclidean Signed Distance, which typically rely on closest‑feature searches and branching logic that hinder SIMD and GPU execution.

We measured runtime on all 15 unordered pairs of the five Platonic solids (tetrahedron, cube, octahedron, dodecahedron, icosahedron). For each pair we generated $N=100$ random poses for these objects. We timed $K=50$ distance queries per configuration and recorded the average per‑query time. Absolute times were obtained for: the exact Euclidean SDF from FCL~\cite{Pan2012}; the unsigned distance, denoted GDF, of \cite{Goncalves2026} (smoothing $h=0.1$, bulging factor $\varepsilon$ uniformly sampled from $[10^{-3},10^{-1}]$);  and our HD-SDF with $\gamma=2$, $\epsilon=10^{-4}$, evaluated on the CPU (C++ single-core) and on a GPU via PyTorch. The PC was equipped with an Intel Core i7-14700 (5.40 GHz), 32 GB RAM, and an NVIDIA GeForce RTX 4060.

Average times are plotted in \Cref{fig:slowdown-comparison} against the complexity measure~\eqref{eq:computational-complexity}, which spans three orders of magnitude. The SDF and GDF lines remain nearly flat because projection‑based methods scale roughly linearly with the vertex count, whereas the HD‑SDF has a fixed cubic complexity. At the same time, projection‑based methods depend on the objects' relative pose, while the HD‑SDF is entirely configuration‑agnostic. On the CPU, the HD‑SDF attains a worst‑case time of approximately $\qty{5}{\milli\second}$, corresponding to a $500\times$ slowdown relative to SDF and a $50\times$ slowdown relative to GDF, consistent with its cubic complexity. Despite this, the C++ implementation remains practical and offers the unique combination of signed output and differentiability. With GPU acceleration, processing a batch of $90^2$ queries, the average per‑query time drops dramatically, yielding a speedup over the CPU implementation of up to $13\times$ for tetrahedron pairs and roughly equivalent times in the worst‑case scenario, demonstrating the throughput gains achievable for large‑scale queries.

\subsection{Conservativeness}\label{sec:experiments-conservativeness}
\begin{figure}[t]
    \centering
    \includegraphics[trim=8px 8px 8px 8px, clip, width=\columnwidth]{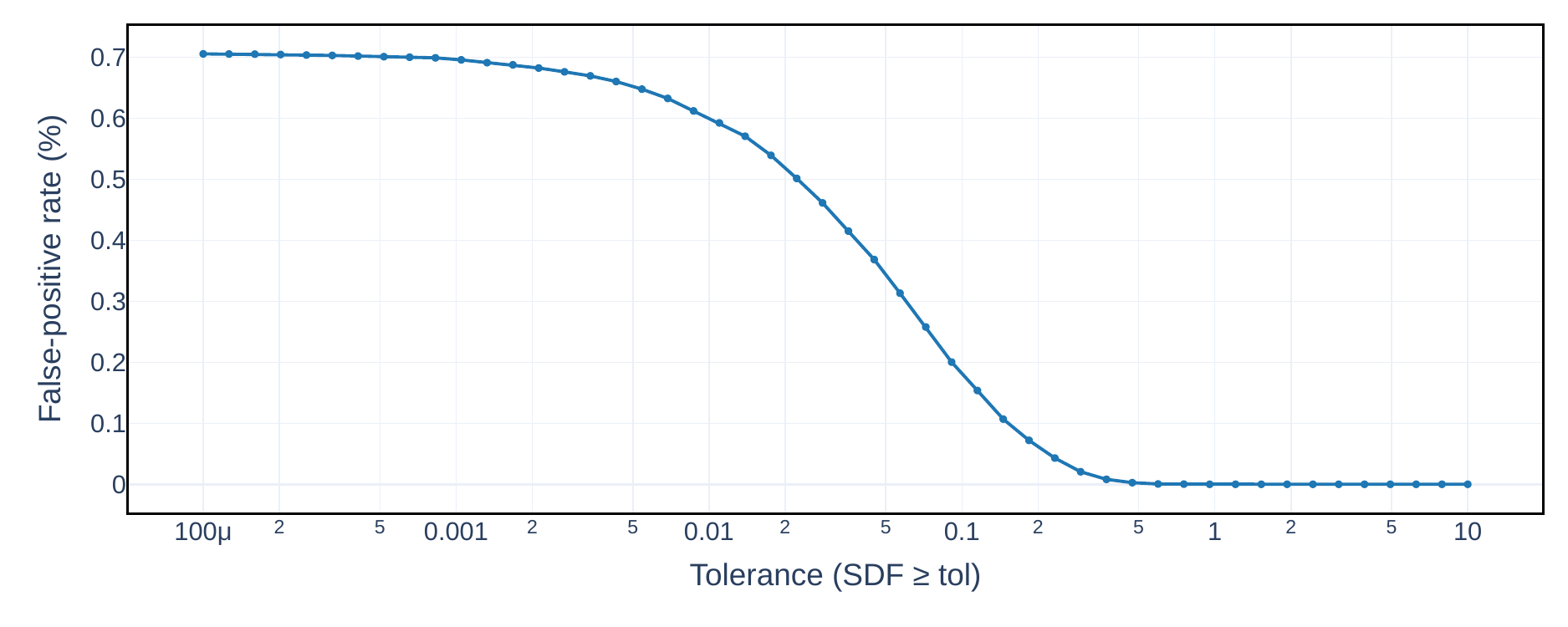}
    \caption{False‑positive rate of the HD‑SDF ($D_{\gamma,\epsilon}\le0$ while the true SDF $d> \texttt{tol}$) as a function of the tolerance $\texttt{tol}$.}
    \label{fig:false-positive}
\end{figure}
\Cref{prop:holder-distance-sign-bound-of-sdf} establishes that $D_{\gamma,\epsilon}>0$ guarantees strict separation, but a non‑positive value may occasionally occur when the true SDF is still positive: a false collision. To quantify this conservativeness, we evaluated $D_{\gamma,\epsilon}$ on $N=\num{500000}$ random configurations drawn from all 15 Platonic‑solid pairs, using the same randomization as in the time comparison. We fixed $\gamma=2$ and $\epsilon=10^{-4}$, and for each configuration compared the sign of $D_{\gamma,\epsilon}$ with the true SDF $d$ from FCL.

\Cref{fig:false-positive} shows the percentage of false positives ($D_{\gamma,\epsilon}\le0$ with $d > \texttt{tol}$) versus the separation tolerance $\texttt{tol}$. For a tolerance of $\qty{10}{\centi\meter}$, the false‑positive rate remains below $0.2\%$; even for a stricter tolerance of $\qty{100}{\micro\meter}$, the rate stays around $0.7\%$. Thus the metric practically never signals contact when the true gap exceeds about $\qty{10}{\centi\meter}$. Conservativeness can be further reduced by adding more continuously‑varying normals to $\tilde{\mathcal{N}}$ (as $\tilde{\mathcal{N}}$ approaches the entire unit sphere, false collisions vanish) but the simple set~\eqref{eq:Ntilde} already provides a very reasonable approximation at moderate cost.

\subsection{Robot Experiment}\label{sec:real-experiment}
\begin{figure}[t]
    \centering
    \includegraphics[width=\columnwidth]{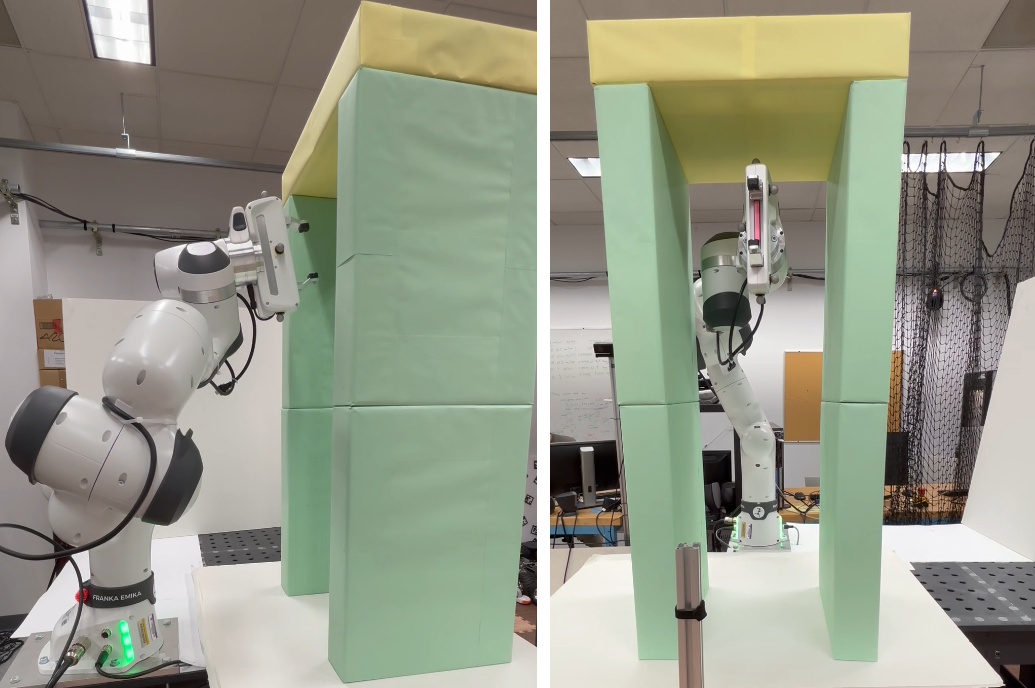}
    \caption{Experimental setup: the Franka Emika Panda entering the corridor (left) and reaching the target pose (right).}
    \label{fig:experiment-snapshots}
\end{figure}
\begin{figure}[t]
    \centering
    \includegraphics[width=\columnwidth]{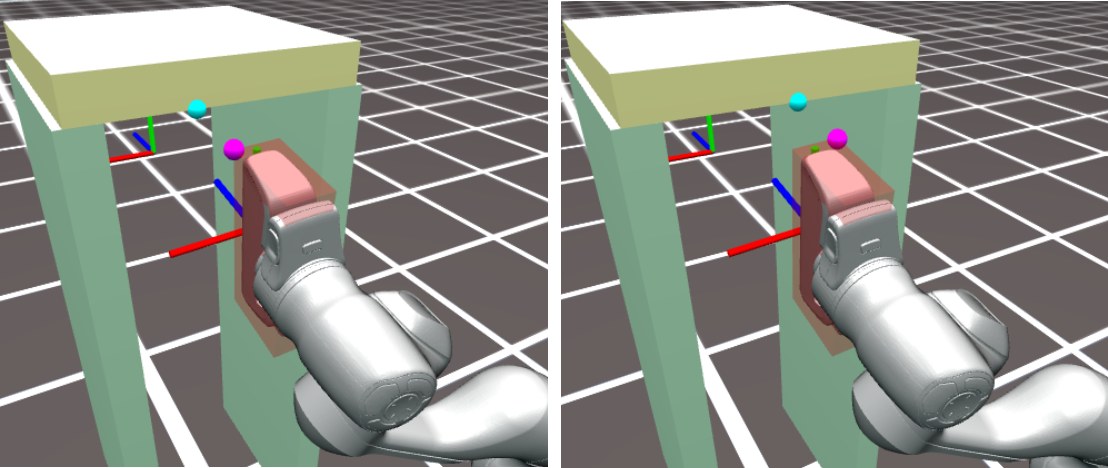}
    \caption{Two snapshots taken $\qty{40}{\milli\second}$ apart during the Euclidean run, showing the witness points (magenta/cyan). The witness points jump due to near‑parallelism, causing non‑differentiability of the distance and control chattering.}
    \label{fig:animation-snapshots}
\end{figure}
We replicate the Franka Emika Panda experiment of \cite{Goncalves2026}, where the robot must reach a target pose inside a corridor of three boxes while respecting collision, self‑collision, and joint‑limit constraints. A CBF‑based quadratic programming controller outputs joint‑velocity commands $u = \dot{q}$ and enforces, for each pair of objects $(i,j)$, the inequality
\begin{align*}
\frac{\partial \Delta_{ij}^{\text{obs}}}{\partial q}(q)  u \geq -\eta_{\text{obs}}\bigl(\Delta_{ij}^{\text{obs}}(q) - \delta_{\text{obs}}\bigr),
\end{align*}
where $\Delta_{ij}^{\text{obs}}$ is the chosen distance (HD‑SDF or Euclidean), $\eta_{\text{obs}}>0$, and $\delta_{\text{obs}}$ is a safety margin. The objective function and joint limit constraints is the same as in \cite{Goncalves2026}. For the Euclidean baseline we use $\delta_{\text{obs}} = \qty{0.03}{\meter}$. For the HD‑SDF we virtually expand each obstacle by $\qty{5e-2}{\meter}$ in all directions and set $\delta_{\text{obs}} = \num{-5e-3}$. This intentionally drives the distance into its negative range, demonstrating that the HD‑SDF remains smooth and differentiable there: a crucial property for control laws that regulate penetration depth, such as insertion, assembly, or compliant manipulation. The expansion creates a virtual safety zone; the negative margin permits slight entry before the CBF activates, effectively aligning the conservative zero‑crossing of the HD‑SDF with the true geometric boundary and preventing the barrier from being overly restrictive. The same CBF structure is applied for self‑collision avoidance (using the respective metric) and joint limits.

\begin{figure}[t]
    \centering
    \includegraphics[width=\columnwidth]{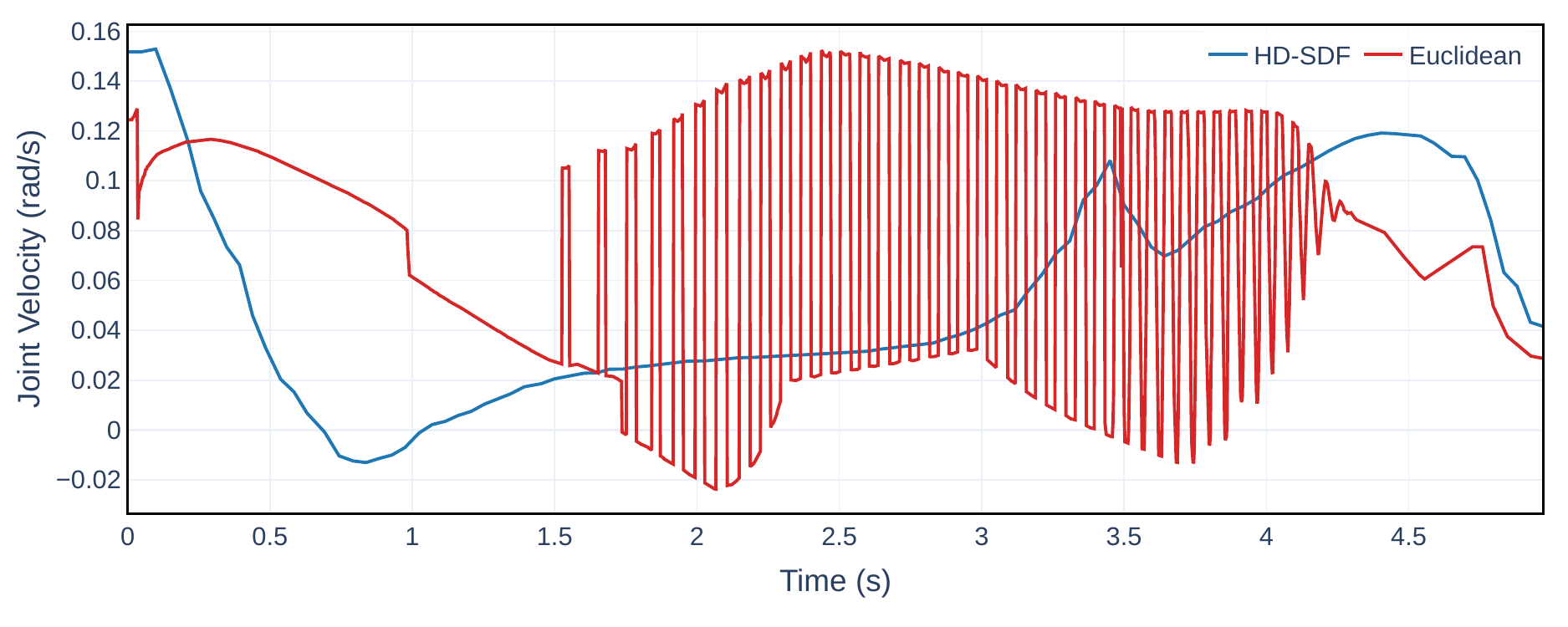}
    \caption{Joint‑velocity command for the last joint during the critical phase ($t = 0$ to $t = 5\,\text{s}$). The Euclidean distance causes strong chattering; the HD‑SDF yields a smoother control input.}
    \label{fig:experiment-control-input}
\end{figure}
All robot links are modeled as boxes, and the box‑based environment frequently creates parallel or nearly parallel faces, causing the Euclidean distance to become non‑differentiable and leading to rapidly switching witness points. This behavior is captured in the snapshots of \Cref{fig:animation-snapshots}: the close‑ups taken just $\qty{40}{\milli\second}$ apart show the magenta witness point jumping from one side of the box to the other. The resulting chattering‑like control inputs are evident for the last joint in \Cref{fig:experiment-control-input}. In contrast, the HD‑SDF produces continuous control signals, and the robot completes the task without vibration. The accompanying video provides a side‑by‑side comparison.

\section{Conclusion}\label{sec:conclusion}
We introduced the Hölder differentiable signed distance function (HD‑SDF), a closed‑form signed distance between convex polyhedra. Its differentiability order is freely chosen through the parameter $\gamma$, so the metric can be made as smooth as a given application requires. The key building blocks (the Hölder minimum and maximum) replace the non‑differentiable min/max of the classical SDF while retaining associativity, commutativity, and sign information. The resulting distance is mildly conservative: it never signals false separation, and the conservatism is small in practice. The metric is non‑iterative and inherently parallelizable; with GPU acceleration, it can be up to $13\times$ faster than a single‑core CPU evaluation for simple shapes and approaches the speed of the Euclidean SDF in the most complex cases. Simulations confirm that the false‑collision rate is negligible for typical safety margins, and a real robot experiment demonstrates smooth CBF‑based control without the chattering caused by the Euclidean distance. Future work will extend the approach to non‑convex shapes. %

\useRomanappendicesfalse
\appendices
\section{}%
\label{appendix:proofs}
\numberwithin{lemma}{section} %

\begin{lemma} \label{lemma:aczel} Let $x = [x_1 \ ... \ x_n]^{\top}$ and $\oplus$ be the $\qmin$. Then:
    \begin{align*}
        \bigoplus_i x_i = \begin{cases}
           \|x\|_{-(\gamma+1)}, & \min_i x_i > 0,\\
         0, & \min_i x_i = 0\\
        -\|(-x)_{+}\|_{\gamma+1}, & \min_i x_i < 0.
        \end{cases}
    \end{align*}
\end{lemma}
\begin{proof}
    By induction.
\end{proof}
\begin{lemma}
\label{lemma:soper}
    Let $y \in \mathbb{R}^m$, define $S_{m,\gamma}(y) \triangleq  \left(\sum_{i=1}^m y_i^{-(\gamma+1)}\right)^{-1}$. Then all the mixed derivatives of order $k \leq \gamma$ vanish at the origin $y=0$.  
\end{lemma}
\begin{proof} 
    It is easy to see that $S_{m,\gamma}(y)$ is $(\gamma+1)$-homogeneous, that is, $S_{m,\gamma}(ty) = t^{\gamma+1}S_{m,\gamma}(y)$ for a scalar $t$. A consequence of Euler Theorem on homogeneous functions (see the only Lemma in \cite{Border2000EulerHomogeneous} is that any $k^{th}$-order mixed partial derivative of this function is  $(\gamma+1-k)$ homogeneous. Thus, using the symbol $S^{(k)}_{m,\gamma}$ for such generic mixed derivative, $S^{(k)}_{m,\gamma}(ty) = t^{\gamma+1-k}S^{(k)}_{m,\gamma}(y)$. The proof follows by setting $t=0$ to conclude $S^{(k)}_{m,\gamma}(0)=0$, which is possible if $k \leq \gamma$.
    \end{proof}
\begin{proof}[\cref{lem:smooth-min-oplus-basic-properties}]
    Associativity follows by inspection and considering the $2^3=8$ possible sign values for $x,y$ and $z$. The remaining properties are immediate.
\end{proof}
\begin{proof}[\cref{lem:smooth-min-differentiable}]
    The operation $(x_1,x_2) \mapsto x_1 \oplus x_2$ is continuous (easily verified from its expression). As $\oplus$ is associative (\Cref{property:smooth-min-oplus-associative}), $f$ is also continuous.

    We prove differentiability separately on the two open regions that make up $\mathcal{D}$.
    
    When $\min_i x_i > 0$, we use \Cref{lemma:aczel} and use the fact---which can be checked by inspection---that $p$-norms $\|x\|_p$ with $p<0$ are infinitely differentiable except when at least one of the $x_i = 0$. Since $\min_i x_i > 0$, differentiability holds.

    The case $\min_i x_i < 0$ is slightly more complex because the map $x \mapsto (x)_+$ is not even once differentiable. However, the map $x \mapsto \|(-x)_+\|^{\gamma+1}_{\gamma+1}$ is $\gamma$-times differentiable everywhere. Applying the outer function $-(\cdot)^{1/(\gamma+1)}$, recovering $-\|(-x)_+\|_{\gamma+1}$, preserves this differentiability except when $\|(-x)_+\|^{\gamma+1}_{\gamma+1}=0$. This happens when $\min_i x_i = 0$. But since $\min_i x_i < 0$, the condition is not met and $f$ is $\gamma$-times differentiable.
\end{proof}
\begin{proof}[\Cref{lemma:diffcompos}]
    The statement in the open set $f(x) \neq 0$ is trivial, due to \Cref{lem:smooth-min-differentiable} and the differentiability of $\phi$.
    
    We now analyze the function $g(x) = \phi\big(f(x)\big)$ around $f(x) = 0$. By the properties of a $\gamma^{th}$ order shaping function, we have the expansion $g(x) = C\sgn(f(x))|f(x)|^{\gamma+1} + o \left(  |f(x)|^{\gamma+2}\right)$, with a constant $C>0$ and where the remainder satisfies $\lim_{f(x) \rightarrow 0} o \left(  |f(x)|^{\gamma+2}\right)/f(x)^r = 0$ for all $r \leq \gamma+1$. For the purpose of studying differentiability at $f=0$, it suffices to work with the dominant term $g(x) = \sgn(f(x))|f(x)|^{\gamma+1}$.
    
    Using \cref{lemma:aczel} and \cref{property:smooth-min-oplus-sign-preservation}, we then have that:
        \begin{align*}
            g(x) = \begin{cases}
               \left( \sum_i 1/x_i^{\gamma+1}\right)^{-1}, & f(x) > 0,\\
             0, & f(x) = 0\\
            - \sum_i (-x_i)_+^{\gamma+1}, & f(x)< 0.
            \end{cases}
        \end{align*}
    We now show that all mixed partial derivatives of order up to $\gamma$ vanish at any point where $f(x)=0$.\footnote{In fact we only need to show that for both expressions (for $f(x)>0$ and $f(x)<0$), the mixed partial derivatives of order up to $\gamma$ agree when evaluated.}
    
    \textit{Case} $f(x) < 0$. Direct differentiation of $g$ in this branch reveals that mixed derivatives of order up to $\gamma$ vanish at a point in which $f(x)=0$.
    
    \textit{Case} $f(x) > 0$ is more intricate. Let $\mathcal{X}$ be the set of all points $x$ with all entries nonnegative and at least one $x_i$ is $0$, which is precisely the zero-level set $f(x)=0$. We need to evaluate the mixed derivatives in the limit when $x$ approaches any $x^* \in \mathcal{X}$. Fix an arbitrary $x^*\in\mathcal{X}$ and partition the indices $\{1,...,n\}$ into two sets: $\mathcal{Z}=\{i\;|\;x_i^*=0\}$ of size $m$ and $\mathcal{P}=\{i\;|\;x_i^*>0\}$ of size ($n-m$). The vector $x$ is then split respectively into vectors $x_Z\in\mathbb{R}^{m}$ and $x_P\in\mathbb{R}^{n-m}$.
    
    Using the function $S$ from \Cref{lemma:soper} the expression for $g$ on $f(x)>0$ can be rewritten as\footnote{If $n=m$, then we define $S_{0,\gamma}(x_P) = \infty$, recovering the original formula $g(x) = S_{n,\gamma}(x)$.}:
    \begin{equation}
        g(x) = \frac{S_{m,\gamma}(x_Z)S_{n-m,\gamma}(x_P)}{S_{m,\gamma}(x_Z)+S_{n-m,\gamma}(x_P)}.
    \end{equation}
    
    Let $g^{(k)}(x)$ denote, generically, any mixed partial derivative of order $k \leq \gamma$. It is easy to see that we can write
    \begin{equation}
    \label{eq:gk}
        g^{(k)}(x) = \frac{R_k(x)}{\big(S_{m,\gamma}(x_Z)+S_{n-m,\gamma}(x_P)\big)^{k+1}}.
    \end{equation}
    The numerator $R_k(x)$ involves mixed partial derivatives of $S_{m,\gamma}$ and $S_{n-m,\gamma}$ from order $0$ to $k$, and it has the following property: if all derivatives of $S_{m,\gamma}(x_Z)$ of order $0$ up to $k$ vanish at a point, then $R_k$ also vanishes at that point, which can be checked by inspection.
    
    We now evaluate at $x^*$. By definition, $S_{m,\gamma}(x_Z^*) = S_{m,\gamma}(0) = 0$ and $S_{n-m,\gamma}(x_P^*) > 0$, hence the denominator in \eqref{eq:gk} is not zero. Furthermore, \Cref{lemma:soper} guarantees that all mixed derivatives of $S_{m,\gamma}(x_Z)$ of order $0 \leq k \leq \gamma$ vanish at $x_Z=x_Z^*=0$, and so does $R_k(x^*)$. This means that $g^{(k)}(x^*)=0$ if $0 \leq k \leq \gamma$. This concludes the proof.
\end{proof}
\begin{proof}[\cref{lem:distance-sign-preservationB}]
    Let $\mathcal{C} \triangleq \mathcal A-\mathcal B$. The SDF uses all  $n$ such that $\|n\|=1$, while $D$ restricts the maximum to a subset $\mathcal{N}$ of this set, hence $D\le d$. We analyze the three sign cases of $d$.
    
    \noindent\textit{Positive.} SAT guarantees an $n^*\in\mathcal{N}$ with $\min_{c\in\mathcal{C}}(n^*)^\top c>0$, thus $d>0\implies D>0$. Conversely, $D>0\implies d \ge D >0$.
    
    \noindent\textit{Negative.} $d< 0\implies D\le d < 0$ immediately. For the reverse, assume $D< 0$ but $d\ge 0$. If $d> 0$, the \textit{positive} case gives $D> 0$, a contradiction. If $d=0$, then $0\in\partial\mathcal{C}$, and by the limiting case of SAT $\exists\, n\in\mathcal{N}$ with $\min_{c\in\mathcal{C}}(n^*)^\top c=0$, so $D\ge0$, again a contradiction. Thus $D<0\implies d<0$.
    
    \noindent\textit{Zero.} The two cases above yield $d=0\iff D=0$.
\end{proof}
    \begin{proof}[\Cref{prop:holder-distance-sign-bound-of-sdf}]
    Let $\tilde{D}(\mathcal{A},\mathcal{B})$ be the metric defined in \eqref{eq:Dtilde}. Clearly, we have that  $\tilde{D}(\mathcal{A},\mathcal{B}) \le d(\mathcal{A},\mathcal{B})$,  since $\tilde{\mathcal{N}}\subset\{n\in\mathbb{R}^3\,\vert\,\|n\|=1\}$

    Thus, from this result, it suffices to prove that $\sgn\bigl(D_{\gamma,\epsilon}(\mathcal{A}, \mathcal{B})\bigr)=\sgn\bigl(\tilde{D}(\mathcal{A}, \mathcal{B})\bigr)$. Let $\mathcal{C}=\mathcal{A}-\mathcal{B}$ and $g_n=\Phi_{\gamma,\epsilon}\left(\bigoplus_{c\in \mathcal{V}(C)} n^\top c\right)$. Then $D_{\gamma,\epsilon}(\mathcal{A},\mathcal{B}) = \Phi_{\gamma,\epsilon}\left(\bigboxplus_{n\in\mathcal{N}} g_n\right)$. By \cref{property:smooth-min-oplus-sign-preservation} and the sign preservation of $\Phi_{\gamma,\epsilon}$: $\sgn(g_n) = \sgn(\min_{c} n^\top c)$. Furthermore, from \cref{property:smooth-max-sign-preservation}, $\bigboxplus$ preserves sign and $\sgn(D_{\gamma,\epsilon}) = \sgn\left(\bigboxplus_{n\in\mathcal{N}} g_n\right) = \sgn(\tilde{D})$.
    \end{proof}%
\begin{proof}[\Cref{lem:differentiability}]
    The argument follows from the differentiability of everything involved and standard arguments about differentiability of composition of functions. 
    
    Enumerate the vertices $a_{i}(t) \in \mathcal{V}(\mathcal{A}(t))$, $b_{j}(t) \in \mathcal{V}(\mathcal{B}(t))$ and  $n_{k}(t) \in \tilde{\mathcal{N}}(\mathcal{A}(t),\mathcal{B}(t))$ for indices $i, j, k$. Define $s_{ijk}(t) \triangleq n_{k}(t)^{\top}\big(a_{i}(t){-}b_{j}(t)\big)$. Since $R_A, R_B, s_A, s_B$ are $\gamma$-times differentiable, $s_{ijk}$ is also $\gamma$-times differentiable.

    Define $g_{k}(t) \triangleq \Phi_{\gamma \epsilon} \big(\bigoplus_{i,j } s_{ijk}(t)\big)$. Since $\Phi_{\gamma,\epsilon}$ is a $\gamma^{th}$ order shaping function (according to \Cref{def:ordershap}), the map $x\mapsto \Phi_{\gamma \epsilon}(\bigoplus_i x_i)$ is $\gamma$-times differentiable in its arguments $x_i$ (\Cref{lemma:diffcompos}). Since the $s_{ijk}$ are $\gamma$-times differentiable in $t$, each $g_k(t)$ is $\gamma$-times differentiable in $t$.

    Finally, $\lambda(t) = \Phi_{\gamma \epsilon}\big(\bigboxplus_k g_k(t) \big)$ is a composition of a shaping function and $\qmax$ over a finite set of $\gamma$-times differentiable functions $g_k$. Using \Cref{def:ordershap} and \Cref{property:smooth-max-differentiable} shows that $\lambda(t)$ is $\gamma$-times differentiable in $t$, completing the proof.
\end{proof}
\begin{proof}[\Cref{thm:holder-differentiable-distance-properties}]
    \Cref{property:HD-SDF-sign-bound-separation-safety} is a direct consequence of \Cref{prop:holder-distance-sign-bound-of-sdf}. \Cref{property:HD-SDF-differentiability} follows from \Cref{lem:differentiability}. For \Cref{property:HD-SDF-rigit-motion-invariance}, note that if $\mathcal{A}, \mathcal{B}$ undergo the same rigid motion, the terms $n^{\top}(a-b)$ remains unchanged.
\end{proof}

\bibliographystyle{ieeetran}
\bibliography{IEEEabrv,refs.bib}
\cleardoublepage

\end{document}